\documentclass{article}

\usepackage[preprint]{neurips_2020}

\usepackage[utf8]{inputenc} % allow utf-8 input
\usepackage[T1]{fontenc}    % use 8-bit T1 fonts
\usepackage{hyperref}       % hyperlinks
\usepackage{url}            % simple URL typesetting
\usepackage{booktabs}       % professional-quality tables
\usepackage{amsfonts}       % blackboard math symbols
\usepackage{nicefrac}       % compact symbols for 1/2, etc.
\usepackage{microtype}      % microtypography

\usepackage{amsmath,amssymb}
\usepackage{color}
\usepackage{tabularx,colortbl}
\usepackage{stackengine}
\usepackage{algorithm}
\usepackage[noend]{algorithmic}
\usepackage{dsfont}
\usepackage{verbatim}
\usepackage{graphicx}
\usepackage{amsthm} % for begin proof
\usepackage[title]{appendix} % for nice appendices
\usepackage{mathabx} % for the \bigtimes
\usepackage{caption}
\usepackage{subcaption}
\usepackage{dblfloatfix}
\usepackage[symbol]{footmisc}
\renewcommand{\thefootnote}{\fnsymbol{footnote}}
\usepackage[colorinlistoftodos]{todonotes}
\usepackage{tikz}
\usetikzlibrary{arrows,automata,positioning}
\usepackage{hyperref}
\usepackage[capitalise]{cleveref}
\usepackage{multirow}
\usepackage{arydshln}
\usepackage{makecell}
\usepackage{wrapfig}

\allowdisplaybreaks

\title{Reward Tweaking: Maximizing the Total Reward While Planning for Short Horizons}

\author{%
  Chen Tessler \\
  Department of Electrical Engineering\\
  Technion Institute of Technology\\
  Haifa, Israel \\
  \texttt{chen.tessler@campus.technion.ac.il} \\
   \And
   Shie Mannor \\
   Department of Electrical Engineering\\
   Technion Institute of Technology\\
   Haifa, Israel \\
   \texttt{shie@ee.technion.ac.il}
}

\newcommand\numberthis{\addtocounter{equation}{1}\tag{\theequation}}

\newtheorem{prop}{Proposition}
\newtheorem{lemma}{Lemma}

\newtheorem{theorem}{Theorem}

\newtheorem{remark}{Remark}

\definecolor{darkgray}{rgb}{0.66, 0.66, 0.66}

\newcommand\SLASH{\char`\\}
\DeclareMathOperator*{\argmax}{arg\,max}

\usepackage{selectp}
\begin{document}

\maketitle

\renewcommand{\thefootnote}{\arabic{footnote}}

\begin{abstract}
    In reinforcement learning, the discount factor $\gamma$ controls the agent's effective planning horizon. Traditionally, this parameter was considered part of the MDP; however, as deep reinforcement learning algorithms tend to become unstable when the effective planning horizon is long, recent works refer to $\gamma$ as a hyper-parameter -- thus changing the underlying MDP and potentially leading the agent towards sub-optimal behavior on the original task. In this work, we introduce \emph{reward tweaking}. Reward tweaking learns a surrogate reward function $\tilde r$ for the discounted setting that induces optimal behavior on the original finite-horizon total reward task. Theoretically, we show that there exists a surrogate reward that leads to optimality in the original task and discuss the robustness of our approach. Additionally, we perform experiments in high-dimensional continuous control tasks and show that reward tweaking guides the agent towards better long-horizon returns although it plans for short horizons.
\end{abstract}

\section{Introduction}

\begin{wrapfigure}{R}{0.5\textwidth}\vspace{-0.4cm}
    \centering
    \includegraphics[width=0.49\textwidth]{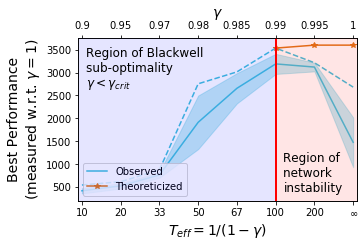}
    \caption{\label{fig: hopper blackwell} Performance on the Hopper-v2 domain as a function of $\gamma$. The dotted line represents the best performing policy across all seeds. Although the agent is trained with $\gamma \in [0.9, 1]$, it is evaluated on $\gamma = 1$. The behavior is split into two regions -- for $\gamma > 0.99$ performance drops as the learning process becomes unstable, whereas for $\gamma \leq 0.99$ performance drops as the agent plans for a short horizon (the discount is small and doesn't satisfy Blackwell optimality).}
    % \vspace{-0.4cm}
\end{wrapfigure}
% The goal of a reinforcement learning (RL) agent is to maximize the reward it can obtain. In the \emph{finite horizon} setting with horizon $T$, one such objective is the total reward $\sum_{t=0}^T r_t$; whereas, in the \emph{infinite horizon} setting we may either consider the average reward $\lim_{T \rightarrow \infty} \frac{1}{T} \sum_{t=0}^T r_t$ or the discounted reward $\sum_{t=0}^\infty \gamma^t r_t$. While the average setting is intuitive, in the discounted case, $\gamma$ controls the planning horizon; specifically, $\gamma$ close to 1 results in behavior that plans for the long term, whereas when $\gamma$ is near 0, the policy becomes greedy and myopic.

Traditionally, the goal of a reinforcement learning (RL) agent is to maximize the reward it can obtain. While there exist a plethora of settings (discounted, average, constrained, etc...), in this work we focus on the \emph{finite horizon} setting with known horizon $T$ and a total reward objective $\max \mathbb{E} \left[\sum_{t=0}^T r_t\right]$. Although the agent is evaluated on the total reward, motivated by recent practical trends, we consider an agent which is trained within the $\gamma$-discounted setting.

% whereas, in the \emph{infinite horizon} setting we may either consider the average reward $\lim_{T \rightarrow \infty} \frac{1}{T} \sum_{t=0}^T r_t$ or the discounted reward $\sum_{t=0}^\infty \gamma^t r_t$. While the average setting is intuitive, in the discounted case, $\gamma$ controls the planning horizon; specifically, $\gamma$ close to 1 results in behavior that plans for the long term, whereas when $\gamma$ is near 0, the policy becomes greedy and myopic.

The discount factor has an important property when training agents using function approximation schemes (neural networks) -- the discount factor acts as a low-pass filter (regularizer \citep{amit2020discount}), effectively reducing the noise in the learning process. Hence, our work focuses on this gap between training and evaluation metrics. Although in most settings, agents are evaluated on the undiscounted total reward, in practice, due to algorithmic limitations, state-of-the-art methods train the agent on a different objective -- the discounted sum \citep{mnih2015human,tessler2017deep,schulman2017proximal,haarnoja2018soft}.

While \emph{Blackwell optimality} \citep{blackwell1962discrete} guarantees that there exists some $\gamma_\text{crit}$ such that solving for $\gamma \geq \gamma_\text{crit}$ results in a policy optimal for the total reward; in practice, finding $\gamma_\text{crit}$ and solving for it, isn't a simple task. As we show in \cref{fig: hopper blackwell}, in practice, the performance of the agent can be characterized by two regions (the location of the border between both regions depends both on the environment and the algorithm): 
\begin{enumerate}
    \item ($\gamma < 0.99$) The well known region of Blackwell sub-optimality $\gamma \ll \gamma_\text{crit}$. Here the planning horizon is insufficient. Long term rewards are discounted and thus have minuscule effect on the policy, resulting in myopic and sub-optimal behavior.
    
    \item ($\gamma > 0.99$) The region of network instability. While, in theory, one would expect behavior as presented in the `theoreticized' line (performance improves with the increase of $\gamma$), as conjectured above, the noise apparent in the learning process is too large and thus as $\gamma$ continues to increase, the performance continues to degrade.
\end{enumerate}

% While a discount sufficiently close to one induces a policy identical to that of the total reward setting \citep{blackwell1962discrete}, a characteristic also known as Blackwell optimality. However, as the discount factor deviates from one, the solution obtained by solving the $\gamma$-discounted problem may become strictly sub-optimal, when evaluated on the total reward. \cref{fig: hopper blackwell} provides a demonstration of this behavior.

% The discounted setting has a very important property -- the contraction of the Bellman operator \citep{banach1922operations}, which enables RL algorithms to find solutions even when function approximators are used (e.g., neural networks). Although in the classic RL formulation, the discount factor is considered as one of the MDP parameters e.g., the task the agent must solve; in practice, it is often considered as a hyperparameter, due to the stability the $\gamma$-discounted setting provides \citep{mnih2015human,kapturowski2018recurrent,xu2018meta}.

Previous work has highlighted the sub-optimality of the training regime and those attempting to overcome these issues are commonly referred to as Meta-RL algorithms. The term ``Meta" refers to the algorithm adaptively changing the objective to improve performance. However, these methods commonly operate in an end-to-end scheme and are thus confined to optimizing the $\gamma$-discounted objective. For instance, Meta-gradients \citep{xu2018meta} optimize the policy on the inner-loss with respect to (w.r.t.) a learned $\gamma_\text{adaptive}$ whereas $\gamma_\text{adaptive}$ is optimized w.r.t. the outer-loss (the original $\gamma=0.99$ objective). An additional work is LIRPG \citep{zheng2018learning}, which learns an intrinsic reward function that, when combined with the extrinsic reward, guides the policy towards better performance in the $0.99$-discounted task.

In this work, we take a different path. Our method, reward tweaking, learns a surrogate reward function $\tilde r$. The surrogate reward is learned such that \textbf{an agent maximizing $\mathbf{\tilde r}$} in the $\mathbf{\gamma}$-discounted setting, will be \textbf{optimal w.r.t. the cumulative undiscounted reward $\mathbf{r}$}. \cref{fig: reward tweaking} illustrates our proposed framework. Thus reward tweaking enables the agent to find optimal policies even while optimizing over short horizons.

We formalize the task of learning $\tilde r$ as a classification problem and analyze it both theoretically and empirically. Specifically, we construct it in the form of finding the max-margin separator between trajectories, a supervised learning problem. We evaluate reward tweaking on both a tabular domain and a high dimensional continuous control task (Hopper-v2 \citep{todorov2012mujoco}). While our method does not mitigate the issues in the region of network instability, it improves the performance of the agent in the region of Blackwell sub-optimality, for all discount values (including at $\gamma = 0.99$ which is observed to perform best within the stable region).

% We observe that by applying reward tweaking, when solving the $\gamma$-discounted objective, the agent is capable of recovering a reward that improves the performance on the undiscounted task. Additionally, results on a tabular domain show that reward tweaking converts sparse reward problems into an equivalent dense reward signal, enabling a dramatic reduction in learning time.

\begin{figure}[t]
    \centering
    \includegraphics[width=0.5\linewidth,trim={10cm 7.3cm 10cm 5.3cm},clip]{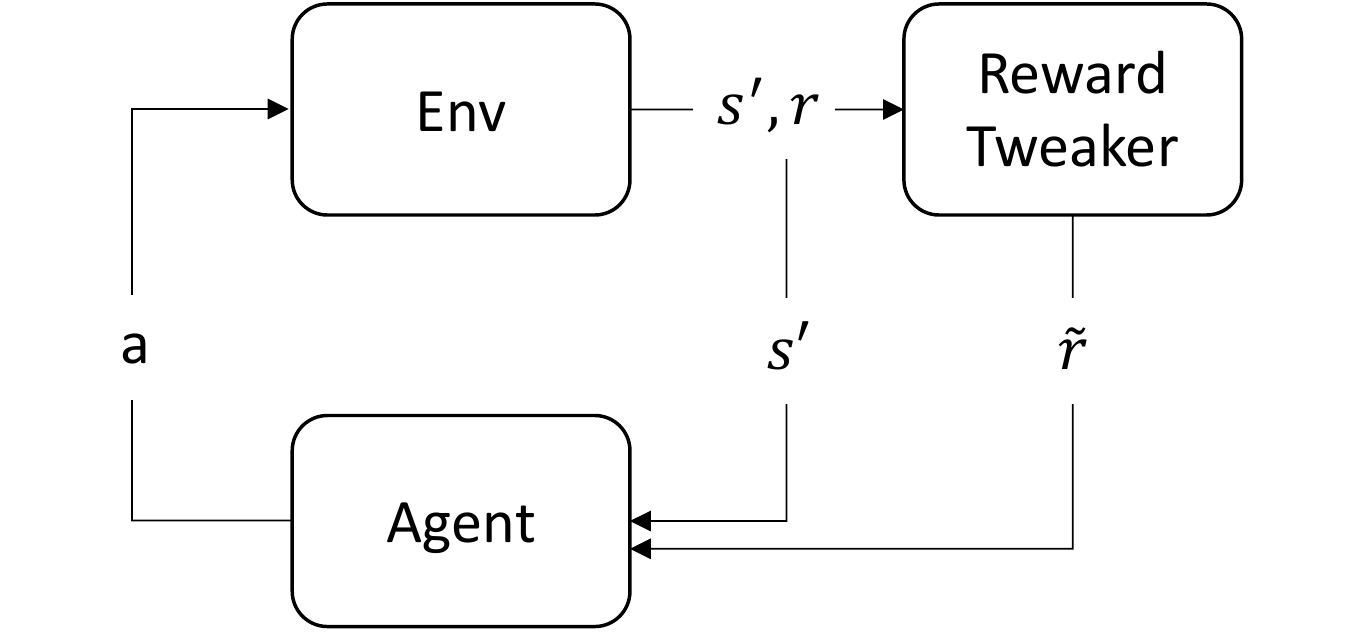}
    \caption{Reward tweaking diagram. Experience collected by the agent is used to train the surrogate reward $\tilde r$. The agent does not observe the reward from the environment, but rather the output of the reward tweaker. The agent attempts to maximize the $\gamma$-discounted reward w.r.t. $\tilde r$. The goal of the reward tweaker is to learn a new reward $\tilde r$ such that the agent will be optimal w.r.t. the original objective $\sum_t r_t$.}
    \label{fig: reward tweaking}
\end{figure}

\section{Related Work}

\textbf{Reward Shaping:} Reward shaping has been thoroughly researched from multiple directions. As the policy is a by-product of the reward, previous works defined classes of reward transformations that preserve the optimality of the policies. For instance, potential-based reward shaping that was proposed by \citet{ng1999policy} and extended by \citet{devlin2012dynamic} that provided a dynamic method for learning such transformations online. While such methods aim to preserve the policy, others aim to shift it towards better exploration \citep{pathak2017curiosity} or attempt to improve the learning process \citep{chentanez2005intrinsically,zheng2018learning}. However, while the general theme in reward shaping is that the underlying reward function is unknown, in \emph{reward tweaking} we have access to the real returns and attempt to tweak the reward such that it enables overcoming the pitfalls of learning with small discount factors.

% \textbf{Meta RL:}
% This work can be seen as a form of meta-learning \citep{andrychowicz2016learning,nichol2018first,xu2018meta}, as the reward learning module acts as a `meta-learner' aiming to adapt the reward signal, guiding the agent towards better performance on the original task. As opposed to prior work, our method does not aim to find a prior that enables few-shot learning \citep{finn2017model}. Rather, it can be seen as a similar approach to \citet{xu2018meta}, which adapts the discount factor in an attempt to increase performance, or \citet{zheng2018learning} that learn an intrinsic reward function that guides the agent towards better policies. While \citet{xu2018meta} and \citet{zheng2018learning} solve and optimize for the discounted setting, we take a different approach -- we aim to learn a reward function for the discounted scenario that guides the agent towards better behavior on the \emph{total reward}.

\textbf{Direct Optimization:} While practical RL methods typically solve the \emph{wrong} task, i.e., the $\gamma$-discounted, there exist methods that can directly optimize the total reward. For instance, evolution strategies \citep{salimans2017evolution} and augmented random search \citep{mania2018simple} perform a finite difference estimation of $\nabla_\theta \mathbb{E}^{\pi_\theta} [\sum_{t=0}^T r_t]$. This optimization scheme does not suffer from stability issues and can thus optimize the real objective directly. As these methods can be seen as a form of finite differences, they require complete evaluation of the newly suggested policies at each iteration, thus, in the general case, these procedures are sample inefficient.

\textbf{Inverse RL and Trajectory Ranking:} 
In Inverse RL \citep{ng2000algorithms,abbeel2004apprenticeship,syed2008game}, the goal is to infer a reward function which explains the behavior of an expert demonstrator; or in other words, a reward function such that the optimal solution results in a similar performance to that of the expert. The goal in Trajectory Ranking \citep{wirth2017survey} is similar, however instead of expert demonstrations, we are provided with a ranking between trajectories. While previous work \cite{brown2019extrapolating} considered trajectory ranking to perform IRL; they focused on a set of trajectories, possibly strictly sub-optimal, which are provided apriori and a relative ranking between them. In this work, the trajectories are collected on-line using the behavior policy and the ranking is performed by evaluating the total reward provided by the environment.

\section{Preliminaries}
We consider a Markov Decision Process \citep[MDP]{puterman1994markov} defined by the Tuple $(\mathcal{S}, \mathcal{A}, \mathcal{R}, \mathcal{P})$, where $\mathcal{S}$ are the states, $\mathcal{A}$ the actions, $\mathcal{R} : \mathcal{S} \to \mathbb{R}$ the reward function and $\mathcal{P} : \mathcal{S} \times \mathcal{A} \to \mathcal{S}$ is the transition kernel.  The goal in RL is to learn a policy $\pi : \mathcal{S} \to \Delta_\mathcal{A}$. While one may consider the set of stochastic policies, it is well known that there exists an optimal deterministic policy. In addition to the policy, we define the value function $v^\pi : \mathcal{S} \to \mathbb{R}$ to be the expected reward-to-go of the policy $\pi$ and the quality function $Q^\pi : \mathcal{S} \times \mathcal{A} \to \mathbb{R}$ the utility of initially playing action $a$ and then acting based on policy $\pi$ afterwards.

In this work we focus on the episodic setting, where each episode is up to length $T$ (where $T$ is known and defined either by the user or the environment), and consider two objectives. (1) The discounted return
$v^\pi_\gamma = \mathbb{E}^{\pi}_{s \sim \mu} \left[\sum_{t=0}^T \gamma^t r_t | s_0 = s\right] \enspace ,$ where $\mu$ is the initial state distribution and the discount factor $\gamma \in [0,1]$ controls the effective planning horizon. (2) The total return $v^\pi_1 = \mathbb{E}^{\pi}_{s \sim \mu, T} \left[ \sum_{t=0}^T r_t | s_0 = s \right] \enspace .$

As we focus on the episodic setting, it is important to note that the optimal policy may not necessarily be stationary. This can easily be solved by appending the remaining horizon to the state \citep{dann2015sample}.

When considering the discounted-episodic return, a well known result \citep{blackwell1962discrete} is that there exists a critical discount factor $\gamma_\text{crit}$, such that $\forall \gamma_\text{crit} \leq \gamma : \argmax_{\pi \in \Pi} v^\pi_\gamma \subseteq \argmax_{\pi \in \Pi} v^\pi_1$. In other words, there exists some high enough discount factor that induces optimality on the total-reward undiscounted objective ($\gamma = 1$).

\section{Maximizing the Total Reward when $\gamma < \gamma_\text{crit}$}
Although the total reward is often the objective of interest, due to numerical optimization issues, empirical methods solve an alternative problem -- the $\gamma$-discounted objective \citep{mnih2015human,hessel2017rainbow,kapturowski2018recurrent,badia2020agent57}. This re-definition of the task enables deep RL methods to solve complex high dimensional problems. However, in the general case, this may converge to a sub-optimal solution \citep{blackwell1962discrete}. We present such an example in \cref{fig:  3 state mdp}, where taking the action left provides a reward $(1 + \gamma) r(a)$ and taking right $r(b) + \gamma r(c)$. For rewards such as $r(a) = 0.5$ and $r(b) = 0, r(c) = 2$ the critical discount is $\gamma_\text{crit} = \frac{1}{3}$. Here, $\gamma > \frac{1}{3}$ results in the policy going to the right (maximizing the total reward in the process) and $\gamma < \frac{1}{3}$ results in sub-optimal behavior, as the agent will prefer to go to the left.

We define a policy $\pi_\gamma$ as the optimal policy for the $\gamma$-discounted task, and the value of this policy on the \emph{total undiscounted reward} as $v^{\pi_\gamma}_1$.
\begin{equation*}
    \pi_\gamma = \argmax_{\pi \in \Pi} \mathbb{E}^\pi \left[ \sum_{t=0}^T \gamma^t r_t \right] \, , \enspace v^{\pi_\gamma}_1 = \mathbb{E}^\pi \left[ \sum_{t=0}^T r_t \right] \enspace .
\end{equation*}
The focus of this work is on learning with \emph{infeasible discount factors}, which we define as the set of discount factors $\{\gamma : \gamma \in [0, \gamma_\text{crit}) \}$ where $\gamma_\text{crit}$ is the minimal $\gamma \in [0, 1]$ such that $||v^{\pi^*} - v^{\pi_\gamma}||_\infty = 0$. Simply put, $\gamma_\text{crit}$ is the minimal discount factor where we can still solve the discounted task and remain optimal w.r.t. the total reward. This is motivated by the limitation imposed by current state-of-the-art algorithms (see \cref{fig: hopper blackwell}), which typically do not work well for discount factors close to 1.

\begin{figure}[t]
    \centering
    \includegraphics[width=0.28\linewidth]{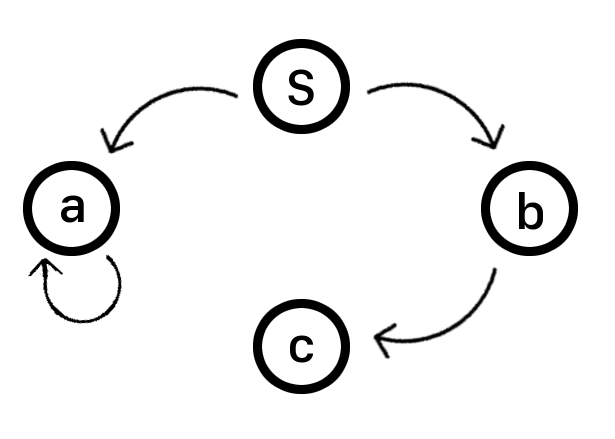}\hfill
    \includegraphics[width=0.58\linewidth]{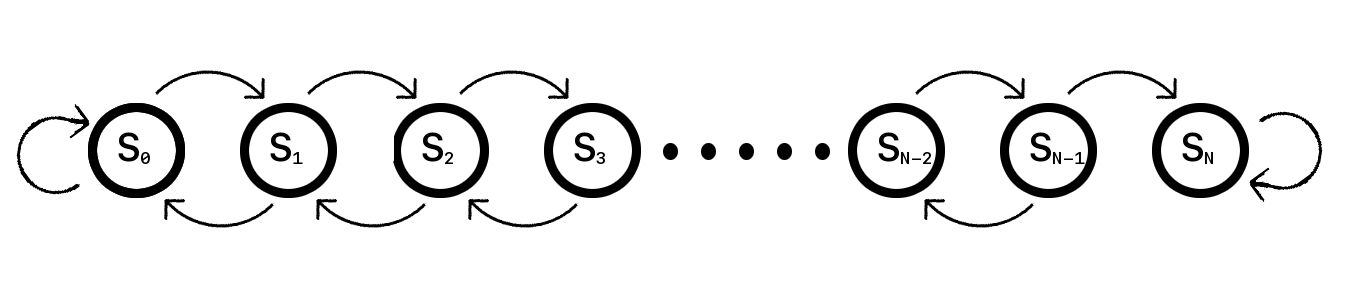}
    \caption{\textbf{Left:} A 4-state MDP with horizon 2. The agent starts at state $s$. Going left provides a reward $(1 + \gamma) r(a)$ and right $r(b) + \gamma r(c)$. For $r(a) = 0.5, r(b) = 0$ and $r(c) = 2$, the critical discount is $\gamma_\text{crit} = \frac{1}{3}$. Hence, when $\gamma > \frac{1}{3}$ the policy prefers to go right (maximizing the total reward) and when $\gamma < \frac{1}{3}$ the agent acts sub-optimally, going to the left. \textbf{Right:} A chain MDP of length N.}
    \label{fig:  3 state mdp}
\end{figure}

While previous methods attempted to overcome this issue through algorithmic improvements, e.g., attempting to increase the maximal $\gamma$ supported by the algorithm \citep{kapturowski2018recurrent}, we take a different approach. Our goal is to learn a surrogate reward function $\tilde r$ such that for any two trajectories $\tau_i, \tau_j$, where $\tau_k = \{ a^k_0, \ldots, a^k_{T-1} \}$, which induce returns ${\mathbb{E}_{\tau^i} [\sum_{t=0}^T \gamma^t \tilde r_t] \geq \mathbb{E}_{\tau^j} [\sum_{t=0}^T \gamma^t \tilde r_t]}$ then ${\mathbb{E}_{\tau^i} [\sum_{t=0}^T r_t] \geq \mathbb{E}_{\tau^j} [\sum_{t=0}^T r_t]}$.

In layman's terms, for a given discount $\gamma$ we aim to find a new reward function $\tilde r$. The agents objective will be to maximize the $\gamma$-discounted reward w.r.t. the new reward $\tilde r$, however, we will continue to evaluate it on the \emph{original} undiscounted reward $\sum_t r_t$. The goal of $\tilde r$ is to guide the agent towards better behavior on the original objective. Finally, it's important to state that this problem is ill-defined. Similar to the results from Inverse RL, there exist an infinite number of reward functions which satisfy these conditions; our goal is to find one of them.

\subsection{Existence of the Solution}
We begin by proving existence of the solution. All proofs are provided in the appendix.

\begin{theorem}\label{lemma: existence}
    For any $\gamma_1, \gamma_2 \in [0,1]$ there exists a reward $\tilde r$ such that
    \begin{align*}
        \tilde \pi^*_{\gamma_1} \in \argmax \mathbb{E}^\pi_{s \sim \mu} [\sum_{t=0}^T \gamma_1^t \tilde r_t | s_0 = s] = \argmax \mathbb{E}^\pi_{s \sim \mu} [\sum_{t=0}^T \gamma_2 r_t | s_0 = s] \ni \pi^*_{\gamma_2} \enspace .
    \end{align*}
    % Where $r$ is the original reward and $T < \infty$ is the horizon.
\end{theorem}

Hence, by finding an optimal policy for the surrogate reward function $\tilde r^i$ in the $\gamma$-discounted case, we recover a policy that is optimal for the total reward.

\begin{remark}\label{remark: non-stationary}
    An important observation is that while the common theme in RL is to analyze MDPs with stationary reward functions; in our setting, as seen in \cref{eqn: existance}, the surrogate reward function is not necessarily stationary. Hence, in the general case, $\tilde r$ must be time dependent.
\end{remark}

Although there exists a reward $\tilde r$ for which $\tilde \pi_\gamma^* = \pi_1^*$ (i.e., optimal w.r.t. the total reward), as the agent is not necessarily provided with this reward apriori, it raises a question as to can this reward be found. The following result tells us that indeed, in the finite horizon setting, such a reward can be found.

\begin{theorem}\label{thm: reward ordering exists}
    For any $\gamma \in (0, 1]$, there exists a mapping $\tilde r (s,t)$ such that for any two trajectories $\tau_i, \tau_j$ and ${\mathbb{E}_{\tau^i} [\sum_{t=0}^T r_t] \geq \mathbb{E}_{\tau^j} [\sum_{t=0}^T r_t]}$ then ${\mathbb{E}_{\tau^i} [\sum_{t=0}^T \gamma^t \tilde r_t] \geq \mathbb{E}_{\tau^j} [\sum_{t=0}^T \gamma^t \tilde r_t]}$.
\end{theorem}

While the above result motivates the online learning procedure, our focus in this work, it holds only for $\gamma$ strictly bigger than $0$. Thus, it many not be possible to reduce all RL tasks to a series of bandit problems (e.g., $\gamma = 0$ results in a myopic agent -- a multi-stage bandit).

\begin{prop}\label{thm: gamma zero doesnt work}
    \cref{thm: reward ordering exists} does not hold for $\gamma = 0$.
\end{prop}

% While \cref{lemma: existence} shows that there exists a solution, learning $\tilde r$ based on \cref{eqn: existance} isn't much different than directly solving the total reward MDP. Thus, if a method is unstable (see \cref{fig: hopper blackwell}) at high discount values, one can not expect such a method to cope with this specific surrogate reward.

% Fortunately, a well-known result from Inverse RL literature \citep{abbeel2004apprenticeship,syed2008game}, this problem is ill-defined (\cref{lemma: non-uniqueness}). As such, we can find alternative rewards which the policy may be capable of learning from.

% \begin{lemma}\label{lemma: non-uniqueness}
%     For any MDP $(\mathcal{S}, \mathcal{A}, \mathcal{R}, \mathcal{P}, T)$, the surrogate reward $\tilde r$ for the $\gamma$-discounted problem is not unique.
% \end{lemma}

% In this work, we operate in the on-policy regime. The reward tweaker observes trajectories sampled from the behavior policy. Hence, the reward does not need to be optimal across all states, but rather along the states visited by the optimal policies. For instance, setting $\gamma = 0$, a tweaked reward is to set $\tilde r(s, a) = 1$ $\forall s, a$ visited by $\pi^*$ and 0 otherwise, i.e., a bandit problem operating only on the active states.

\subsection{Robustness}

In addition to analyzing the existence of the solution, we analyze the robustness gained by learning with an optimal surrogate reward function in the $\gamma$-discounted setting. Here, we assume that the reward is given and the agent solves the finite horizon $\gamma$-discounted reward $\tilde r$, yet is evaluated on the undiscounted reward $r$.

\citet{kearns1998near} proposed the Simulation Lemma, which enables analysis of the value error when using an empirically estimated MDP. By extending their results to the finite horizon $\gamma$-discounted case, we show that in addition to improving rates of convergence, the optimal surrogate reward increases the robustness of the MDP to uncertainty in the transition matrix.

\begin{prop}\label{lemma: simulation}
    If $\tilde r$ enables recovering the optimal policy (e.g., \cref{eqn: existance}), $\max_{s,a} |\mathcal{\hat P}(s,a) - \mathcal{P}(s,a)| \leq \epsilon_P$, where $\hat P$ is the empirical probability transition estimates, and $\tilde R_\text{max} = \max_{s \in \mathcal{S}} \tilde r(s)$, then
    \begin{equation*}
        \forall s \in \mu : || v^\pi_{\hat P} - v^\pi_P ||_\infty  \leq \frac{\gamma (1 - \gamma^T) \tilde R_\text{max}}{2 (1-\gamma)^2} \epsilon_P \enspace .
    \end{equation*}
\end{prop}

As we are analyzing the finite-horizon objective, it is natural to consider a scenario where the uncertainty in the transitions $|\hat P(s' | s, a) - P(s' | s, a)|$ is concentrated at the end of the trajectories. This setting is motivated by practical observations -- the agent commonly starts in the same set of states, hence, most of the data it observes (higher confidence) is concentrated in those regions.

The following proposition presents how the bounds improve by a factor of $\gamma^{T-L}(1-\gamma^L)/(1-\gamma^T)$ when the uncertainty is concentrated near the end of the trajectories, as a factor of both the discount $\gamma$ and the size of the uncertainty region defined by $\mathcal{S}_L$.

\begin{prop}\label{lemma: simulation at the end}
    We assume the MDP can be factorized, such that all states $s$ that can be reached within $T-L \leq t \leq T$ steps are unreachable for $t < T-L$, $\forall \pi \in \Pi$. We denote the ${T-L \leq t \leq T}$ reachable states by $\mathcal{S}_L$.
    
    If $\max_{s \in \mathcal{S}_L, a} ||\mathcal{\hat P}(s,a) - \mathcal{P}(s,a)||_1 \leq \epsilon_P$ and $\max_{s \in \mathcal{S}\SLASH\mathcal{S}_L, a} ||\mathcal{\hat P}(s,a) - \mathcal{P}(s,a)||_1 = 0$ then
    \begin{equation*}
        \forall s \in \mu : || v^\pi_{\hat P} - v^\pi_P ||_\infty  \leq \gamma^{T-L} \frac{\gamma (1 - \gamma^L) \tilde R_\text{max}}{2 (1-\gamma)^2} \epsilon_P \enspace .
    \end{equation*}
\end{prop}

% \subsection{The $\gamma$ trade-off}\label{sec: lambda tradeoff}

% It is well-known that as $\gamma$ decreases, the sample complexity of the algorithm decreases and the convergence rate increases \citep{petrik2009biasing,strehl2009reinforcement,jiang2015dependence}. This raises an immediate question -- if we can learn an optimally tweaked reward, for any $\gamma \in [0, \gamma_\text{max})$, where $\gamma_\text{max}$ is the maximal discount the algorithm is capable of learning with, why not focus on where it is easiest to learn the policy, i.e., $\gamma = 0$? Obviously, there exists a trade-off between the complexity of learning the policy and that of learning a tweaked reward that leads to the optimal policy.

% Consider the chain MDP in \cref{fig:  3 state mdp}, where the goal is to reach $S_N$ as fast as possible. The game ends when either (i) $S_N$ is reached, or (ii) $T$ time-steps have passed. In this case, the trade-off is clear; when $\gamma > 0$, $\forall s \in \mathcal{S}\SLASH s_N : r(s) = 0$ and $r (s_N) > 0$ is a solution -- yet finding an optimal policy in this case is a complex task. On the other hand, while for $\gamma = 0$ the rewards need to be constructed at each state and is thus a hard task; given these rewards, learning the policy is relatively easy. This shows that there exists a trade-off and the optimal discount will reside within $(0, 1)$.

\section{Learning the Surrogate Objective}

Previously, we have discussed the existence and benefits provided by a reward function, which enables learning with smaller discount factors. In this section, we focus on how to learn the surrogate reward function $\tilde r$. %When the model is known, a non-stationary variant of $\kappa$-Policy Iteration \citep[$\kappa$-PI]{efroni2018beyond} can be used to learn the reward function $\tilde r$. In addition, when the undiscounted task can be solved, the reward can be constructed as shown in \cref{eqn: existance}.

% However, an important component of $\kappa$-PI is an inner loop that evaluates the policy using the original discount value. As we consider a scenario where the algorithm fails in the original task, i.e., we can only estimate the value for suboptimal values of $\gamma$, it will also fail in the evaluation phase. Hence, we opt for an alternative, iterative optimization scheme. 

With a slight abuse of notation, we define the set of $K$ previously seen trajectories $\{ \tau_i \}_{i=0}^K = \{ (r_0^{(i)}, s_0^{(i)}), \ldots, (r_T^{(i)}, s_T^{(i)}) \}_{i=0}^K$ and a sub-trajectory by $\tau_i^{(t)} = \{ (r_t^{(i)}, s_t^{(i)}), \ldots, (r_T^{(i)}, s_T^{(i)}) \}$. Our goal is to find a reward function $\tilde r : \mathcal{S} \mapsto \mathbb{R}$ that satisfies the ranking problem, for any $i,j$ and $n,m \in [0, T-1]$. That is, assuming $\sum_{t=n}^T r^{(i)}_t > \sum_{t=m}^T r^{(j)}_t$, learning $\tilde r$ is done by solving:
\begin{equation}\label{eqn: loss}
    \mathcal{L}(\tilde r; \tau_i^n, \tau_j^m) = \frac{e^{\left(\sum_{t=n}^T \gamma^{t-n} \tilde r \left(s_t^{(i)}\right) \right)}}{e^{\left(\sum_{t=n}^T \gamma^{t-n} \tilde r\left(s^{(i)}_t\right)\right)} + e^{\left(\sum_{t=m}^T \gamma^{t-m} \tilde r\left(s^{(j)}_t\right)\right)}} \, .
\end{equation}

\begin{remark}\label{rem: tabular}
    For a tabular state space, we may represent a trajectory $\tau_i$ by $\phi(\tau_i) = \sum_{t=0}^T \gamma^t \mathds{1} (s = s_t) \in \mathbb{R}^{|\mathcal{S}|}$, and $\tilde r$ as a linear mapping. This is a logistic regression \citep{soudry2018implicit}, where the goal is to find the max-margin separator between trajectories 
    \begin{equation}
        \mathcal{L}(\tilde r; \tau_i^n, \tau_j^m) = \frac{e^{\left(\phi^T(\tau_i^n) \tilde r\right)}}{e^{\left(\phi^T(\tau_i^n) \tilde r\right)} + e^{\left(\phi^T(\tau_j^m) \tilde r \right)}} \enspace .
    \end{equation}
    %Since this is a linear separator, and we know that there exists a solution, then following the results from \citet{soudry2018implicit}, we know that $\tilde r$ will indeed converge to the max-margin separator. In addition, \citet{lyu2020gradient} showed that under mild assumptions, deep homogeneous networks converge to some \emph{local} margin-maximizing solution, this suggests that one may expect this method to work also in a deep RL regime.
\end{remark}

As we are concerned with maximizing the margin between all trajectories, the task becomes

\begin{equation}\label{eqn: diff objective}
    \min_{\tilde r \in \mathcal{R}} \mathbb{E}_{\tau_i, \tau_j} \left[ \sum_{n,m \in [0,T-1]} \mathcal{L} (\tilde r; \tau_i^n, \tau^m_j) \right] \enspace .
\end{equation}

As opposed to \citet[T-REX]{brown2019extrapolating}, we are not provided with apriori demonstrations and a ranking between them, instead, we follow an online scheme (\cref{alg: alg} and \cref{fig: reward tweaking}). As a part of the training phase, the trajectories observed by the policy $\pi$ are provided to the reward learner. The reward learner provides the most recent learned reward $\tilde r$ to the policy. As shown in \cref{thm: reward ordering exists}, this ranking problem (\cref{eqn: diff objective}) has a solution; thus (informally), assuming the policy has a non-zero support on each trajectory (i.e., each trajectory is observed infinitely often), \cref{alg: alg} will converge (in the limit) to the optimal policy.

% \begin{remark}\label{rem: action gap}
%     As \cref{eqn: diff objective} aims to maximize the margin between trajectories, this increases the action-gap \citep{farahmand2011action,bellemare2016increasing}, potentially enabling more efficient learning in low discount scenarios.
% \end{remark}

\begin{algorithm}[t]
	\caption{Adaptive Reward Tweaking}
	\label{alg: alg}
	\begin{algorithmic}
		\STATE {\bf Initialize} surrogate reward $\tilde r (s)$, policy $\pi$ and replay buffer $\mathcal{B}$
		\FOR{$i = 0, \ldots$}
    		\STATE Play trajectory $\tau_i = \{ (s_0^i, a_0^i \sim \pi (s_0^i), r_0^i), \ldots \}$
    		\STATE Append trajectory $\tau_i$ to buffer $\mathcal{B}$
    		\STATE Sample batch of trajectories $\{ \tau_k^{j_k} | j_k \sim |\tau_k| \}_{k=0}^N \sim \mathcal{B}$
    		\STATE Minimize $\sum_{k \neq h} \mathcal{L} (\tilde r; \tau_k^{j_k}, \tau_h^{j_h})$
    		\STATE Optimize $\pi$ w.r.t. $\tilde r$
		\ENDFOR
		\STATE {\bf Return} $\pi$
	\end{algorithmic}
\end{algorithm}

\section{Experiments}

% We focus on two domains, a tabular puddle world and a robotic control task. While the tabular domain enables analysis of reward tweaking, through visualization of the reward as the discount factor changes; the robotic task showcases the ability of reward tweaking to work in high dimensional tasks where function approximation is needed.

% For the robotics task, we focus on the Hopper domain, where current methods can find near-optimal solutions. We show that (1) for high discount factors the algorithms fail, even when solving a finite-horizon task with an appropriate non-stationary model, and (2) that reward tweaking is capable of guiding the agent towards optimal performance, even when the effective planning horizon is reduced dramatically. Hence, for domains where the algorithms maximal planning horizon is smaller than $\gamma_\text{crit}$, reward tweaking can be lead to large performance gains.

\subsection{Tabular Reinforcement Learning}

To understand how reward tweaking behaves, we analyze a grid-world, depicted in \cref{fig: puddle}. Here, the agent starts at the bottom left corner and is required to navigate towards the goal (the wall on the right). Each step, the agent receives a reward of $r_\text{step} = -1$, whereas the goal, an absorbing state, provides $r_\text{goal} = 0$. Clearly, the optimal solution is the stochastic shortest path. However, at the center resides a puddle where $r_\text{puddle} = -0.5$. The puddle serves as a \emph{distractor} for the agent, creating a critical point at $\gamma_\text{crit} \approx 0.7$. While the optimal \emph{total reward} is $-6.5$, an agent trained \emph{without} reward tweaking for $\gamma < \gamma_\text{crit}$ prefers to reside in the puddle and obtains a \emph{total reward} of $-11$.

For analysis of plausible surrogate rewards, obtained by \emph{reward tweaking}, we present the heatmap of a reward in each state, based on \cref{eqn: existance}. As the reward is non-stationary, we plot the reward at the minimal reaching time, e.g., $\tilde r (s, d(s, s_0))$ where $d$ is the Manhattan distance in the grid. The arrows represent the gradient of the surrogate reward $\nabla_s \tilde r(s)$, i.e., the policy of a myopic agent. Notice that when $\gamma = 1$, the default reward signal is sufficient. However, in this domain, when the discount declines below the critical point, the 1-step surrogate reward (e.g., $\gamma = 0$) becomes informative enough for the myopic agent to find the optimal policy (the arrows represent the myopic policy). In other words, reward tweaking with small discount factors leads to denser rewards, resulting in an easier training phase for the policy.

\begin{figure*}[t]
    \centering
    \includegraphics[width=0.32\linewidth,trim={1cm 0.6cm 1cm 0.8cm},clip]{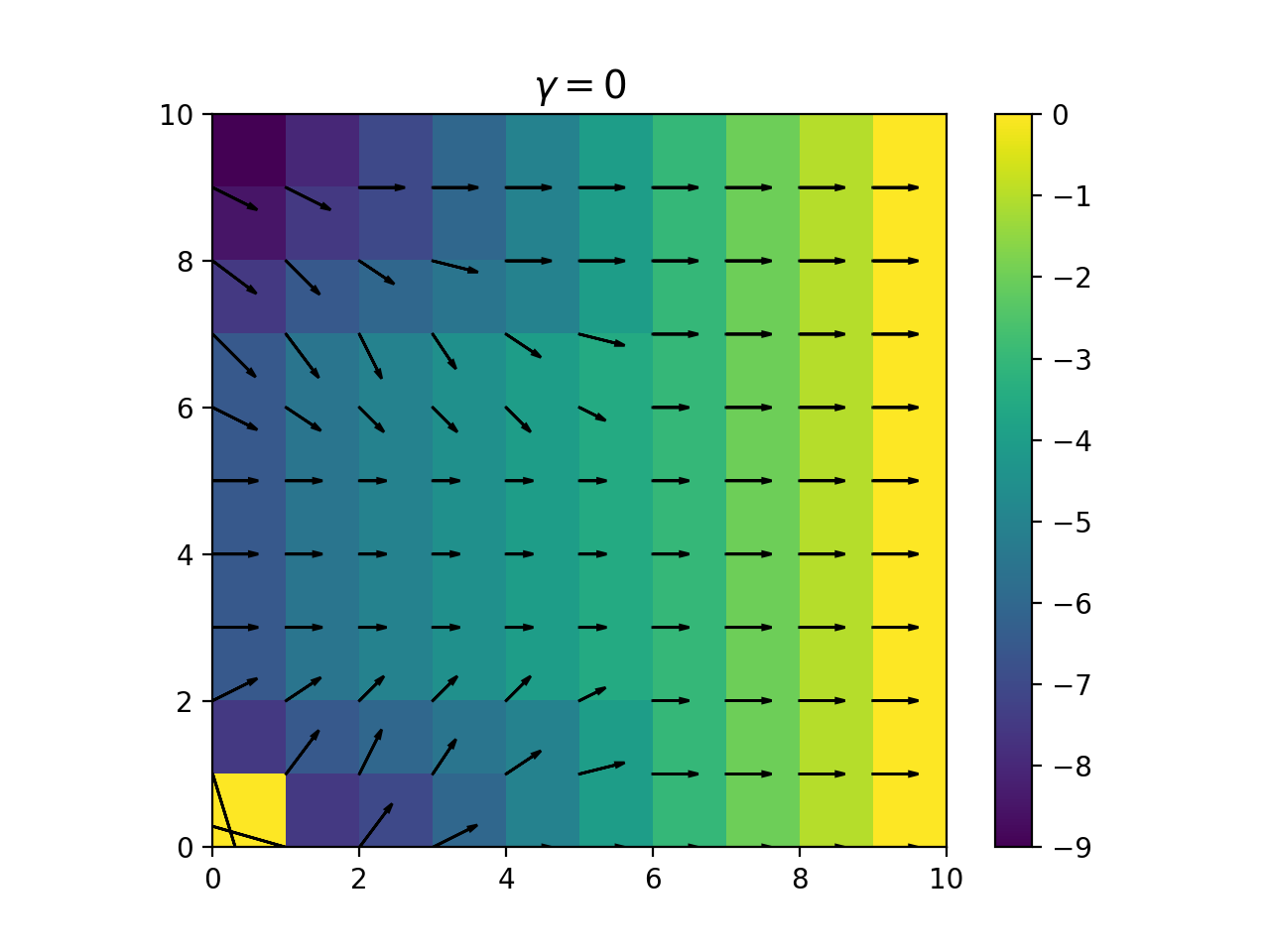}
    \includegraphics[width=0.32\linewidth,trim={1cm 0.6cm 1cm 0.8cm},clip]{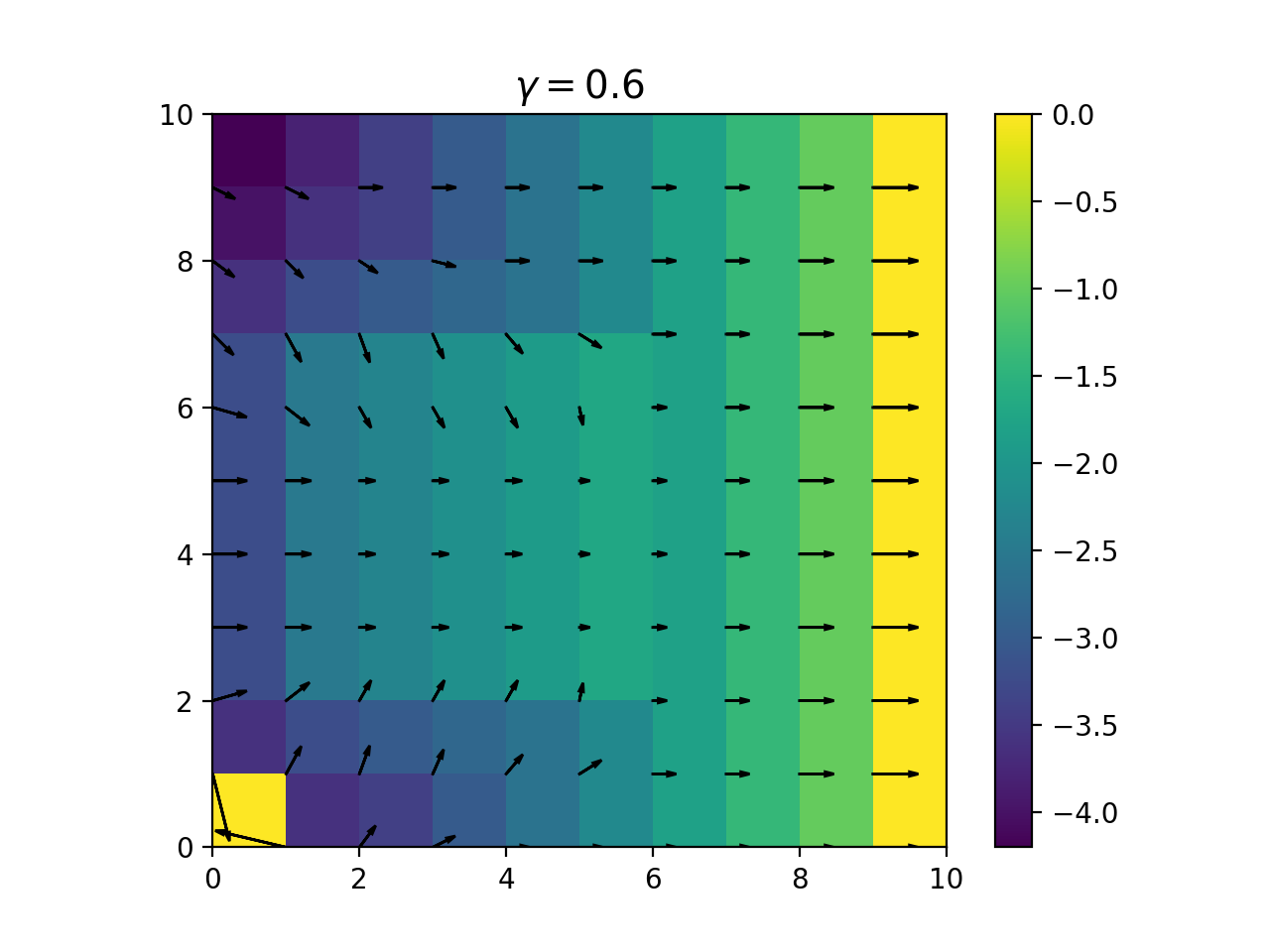}
    \includegraphics[width=0.32\linewidth,trim={1cm 0.6cm 1cm 0.8cm},clip]{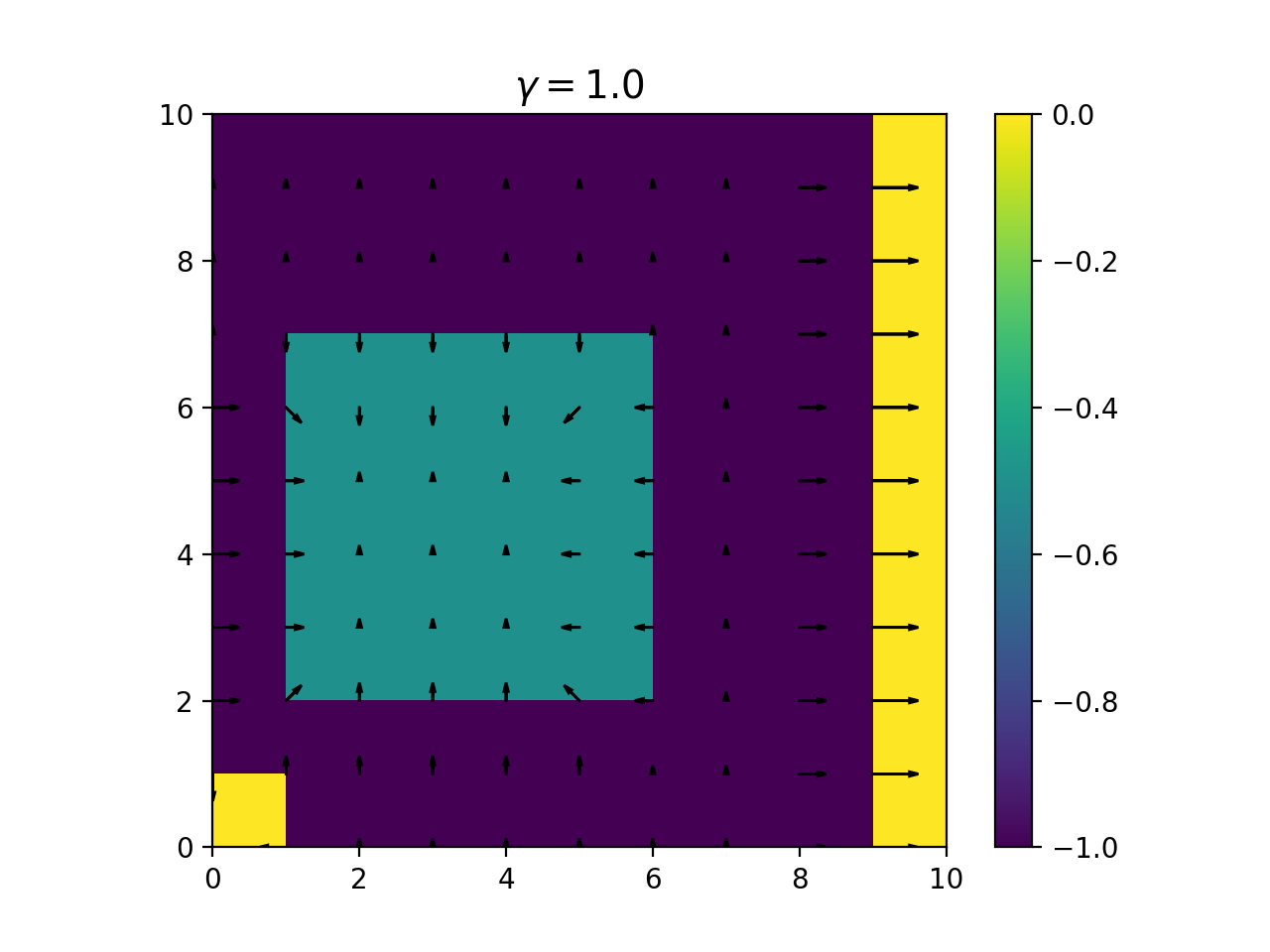}
    \caption{The surrogate reward in the puddle world domain as a function of $\gamma$. The agent starts at the bottom left. The objective is to reach the right wall. The puddle provides a small reward, thus acts as a \emph{distractor} when $\gamma < \gamma_\text{crit}$, attracting the agent to remain inside it. The arrows represent the policy of a myopic agent (solving for $\gamma = 0$), i.e., the gradient of the reward function.}
    \label{fig: puddle}
\end{figure*}

\subsection{Deep Reinforcement Learning}

\begin{wrapfigure}{R}{0.5\textwidth}\vspace{-0.75cm}
    \centering
    \includegraphics[width=0.49\textwidth]{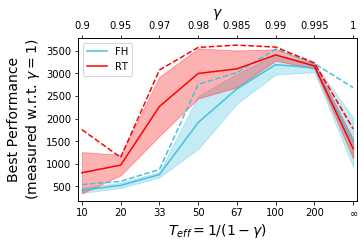}
    \caption{Hopper-v2 experiments, comparing the TD3 finite horizon baseline [FH] with our method, Reward Tweaking [RT]. The top X axis denotes the discount factor $\gamma$, and the bottom presents the effective planning horizon $T_{eff} = 1/(1-\gamma)$. The dashed line shows the max performance over all seeds, the solid is the mean and the colored region is the 95\% confidence bound.}
    \label{fig: hopper pareto}
    \vspace{-0.6cm}
\end{wrapfigure}

In addition to analyzing reward tweaking in the tabular domain, we evaluate it on a complex task, using approximation schemes (deep neural networks). We focus on Hopper \citep{todorov2012mujoco}, a finite-horizon ($T=1,000$) robotic control task where the agent controls the torque applied to each motor. We build upon the TD3 algorithm \citep{fujimoto2018addressing}, a deterministic off-policy policy-gradient method \citep{silver2014deterministic,lillicrap2015continuous}, that combines two major components -- an actor and a critic. The actor maps from state $s$ to action $a \in [-1,1]^d$, i.e., the policy; and the critic $Q(s,a)$ is trained to predict the expected return. Our implementation is based on the original code provided by \citet{fujimoto2018addressing}, including their original hyper-parameters. For each value of $\gamma$, we compare 5 randomly sampled seeds, showing the mean and $95\%$ confidence interval of the mean. This results in a total of 40 seeds.

\subsubsection{Experimental Details}
We compare reward tweaking with the baseline TD3 model, adapted to the finite horizon setting. We provide additional experiments details in the supplementary material, in addition to numerical results. Graphical results are provided in \cref{fig: hopper pareto}.

As we are interested in a finite-horizon task, the reward tweaking experiments are performed on a non-stationary version of TD3, i.e., we concatenate the normalized remaining time $\tilde t = (T-t)/T$ to the state. In our experiments\footnote{Code accompanying this paper is provided in \href{https://bit.ly/2O1klVu}{bit.ly/2O1klVu}}, $\tilde r$ and $\pi$ are learned in parallel (\cref{fig: reward tweaking}). In order to learn $\tilde r$, we sample trajectories up to length of $L$ steps. This is a computational consideration -- a larger $L$ will lead to better performance, but will result in a longer run-time. Experimentally, we found $L=100$ to suffice. As $\tilde r$ is trained on $L \ll T$, the model is required to extrapolate for $L < t \leq T$ (a similar trade-off exists in standard back-propagation through time \citep{sutskever2013training}).

\definecolor{Gray}{gray}{0.95}
\newcolumntype{g}{>{\columncolor{Gray}}c}

\subsubsection{Results}
We present the results in \cref{fig: hopper pareto} and the numerical values in the supplementary material. For each discount and model, we run a simulation using 5 random seeds. For each seed, at the end of training, we re-evaluate (over $100$ episodes) the policy that is believed to be best. The results we report are the average, STD and the max of these values (over the simulations performed).

\textbf{Baseline analysis:} The baseline behaves best at $\gamma \approx 0.99$. This leads to two important observations. (1) Although the model is fit for a finite-horizon task, and the horizon is indeed finite, it fails when the effective planning horizon increases above $100$ ($\gamma > 0.99$). This suggests that practical algorithms have a maximal feasible discount value, such that for larger values they suffer from performance degradation. (2) As $\gamma$ decreases, the agent becomes myopic leading to sub-optimal behavior.

These results motivate the use of reward tweaking. While reward tweaking is incapable of overcoming the first problem, it is designed to mitigate the second -- the failure which occurs due to not satisfying the Blackwell optimality criteria \citep{blackwell1962discrete}.

\textbf{Reward Tweaking:} (1) As hypothesised, reward tweaking improves the performance across all feasible discount values, i.e., values where TD3 is capable of learning ($\gamma \leq 0.99$). (2) Learning the reward becomes harder as $\gamma$ decreases, hence the variance across seeds grows as the planning horizon decreases. Finally, (3) focusing on the maximal performance across seeds, we see that while this is a hard task, reward tweaking improves the performance in terms of average and maximal score -- and even improves the average performance at $\gamma = 0.99$ (the maximal feasible discount for the TD3 baseline on this task) from $\approx 3200$ to $\approx 3400$ ($+6\%$ performance improvement).

\section{Conclusions}

In this work we tackled the problem of learning with Blackwell \emph{sub}-optimality; the algorithm aims to maximize the total reward, yet is limited to learning with sub-optimal discount factors. We proposed \emph{reward tweaking}, a method for learning a reward function which induces better behavior on the original task. Our proposed method enables the learning algorithm to find optimal policies, by learning a new reward function $\tilde r$. We showed, in the tabular case, that there indeed exists such a reward function. In addition, we proposed an objective to solving this task, which is equivalent to finding the max-margin separator between trajectories.

We evaluated reward tweaking in a tabular setting, where the reward is ensured to be recoverable, and in a high dimensional continuous control task. In the tabular case, visualization enables us to analyze the structure of the learned rewards, suggesting that reward tweaking learns dense reward signals -- these are easier for the learning algorithm to learn from. In the Hopper domain, our results show that reward tweaking is capable of outperforming the baseline across most planning horizons. Finally, we observe that while learning a good reward function is hard for short planning horizons, in some seeds, reward tweaking is capable of finding good reward functions. This suggests that further improvement in stability and parallelism of the reward learning procedure may further benefit this method.

Although in terms of performance reward tweaking has many benefits, this method has a computational cost. Even though learning the surrogate reward does not require additional interactions with the environment (sample complexity), it does require a significant increase in computation power (wallclock time). As we sample partial trajectories and compute the loss over the entirety of them, this requires many more computations when compared to the standard TD3 scheme. Thus, it is beneficial to use reward tweaking when the performance degradation, due to the use of infeasible discount factors, is large, or when sample efficiency is of utmost importance.

We focused on scenarios where the algorithm is incapable of solving for the correct $\gamma$. Future work can extend reward tweaking to additional capabilities. For instance, to stabilize learning, many works clip the reward to $[-1, 1]$ \citep{mnih2015human}. However, these changes, like the discount factor, change the underlying task the agent is solving. Reward tweaking can be used to find rewards which satisfy these constraints, yet still induce optimal behavior on the total original reward.

Finally, while reward tweaking has a computational overhead, it comes with significant benefits. Once an optimal surrogate reward has been found, it can be re-used in future training procedures. As it enables the agent to solve the task using smaller discount factors, given a learned $\tilde r$, the policy training procedure is potentially faster.

\clearpage

\section{Broader Impact}
Reward tweaking tackles a fundamental problem in applied reinforcement learning. As reinforcement learning algorithms commonly optimize the discounted infinite horizon objective, they underperform on the required metric -- the undiscounted total reward.

Reward tweaking bridges the gap between the objective we want to solve and that which our algorithms are capable of solving. As such, reward tweaking has the potential to impact reinforcement learning on a broad scale across almost all application levels. These applications range from physical applications (e.g., robotics, autonomous vehicles, warehouse automation, and more) and digital applications (e.g., marketing, recommendation systems, and more).

%
% ---- Bibliography ----
%
\bibliographystyle{plainnat}
\bibliography{main.bib}

\appendix

\section{Proofs}

\begin{theorem}
    For any $\gamma_1, \gamma_2 \in [0,1]$ there exists $\tilde r$ such that
    \begin{align*}
        \tilde \pi^*_{\gamma_1} &= \argmax \mathbb{E}^\pi_{s \sim \mu} [\sum_{t=0}^T \gamma_1^t \tilde r_t | s_0 = s] \\
        &= \argmax \mathbb{E}^\pi_{s \sim \mu} [\sum_{t=0}^T \gamma_2 r_t | s_0 = s] = \pi^*_{\gamma_2} \enspace .
    \end{align*}
    Where $r$ is the original reward and $T < \infty$ is the horizon.
\end{theorem}

\begin{proof}
We define the optimal value function for the finite-horizon total reward by $v^{i,*}$, the optimal policy $\pi^{i,*} (s) = \argmax_{a \in \mathcal{A}} \sum_{s' \in \mathcal{S}} \mathcal{P} (s' | s, a) v^{i+1,*} (s')$ and the value of the surrogate objective $\tilde r$ with discount $\gamma$ by $\tilde v^{i, \pi}_{\gamma}$. Thus, extending the result shown in \citet[$\kappa$-PI]{efroni2018beyond}, we may define the surrogate reward as
\begin{equation}\label{eqn: existance}
    0 \leq i < \tau : \tilde r^i = r + (1 - \gamma) \mathcal{P}^{\pi^*} v^{i+1,*} \enspace .
\end{equation}
Applying the Bellman operator, we get
\begin{align*}
    &\tilde v^{i, \pi}_{\gamma_2} (s) = \tilde r^i (s) + \gamma_2 \sum_{s' \in \mathcal{S}} \mathcal{P}(s'| s, \pi(s)) v^{i+1,\pi} (s') \\
    &= r(s) + \left(\gamma_1 - \gamma_2\right) \sum_{s' \in \mathcal{S}} P^{\pi^*_{\gamma_1}} (s' | s, \pi^*_{\gamma_1} (s)) v^{i+1,\pi^*_{\gamma_1}} (s') \\
    &\enspace\enspace+ \gamma_2 \sum_{s' \in \mathcal{S}} \mathcal{P}(s'| s, \pi(s)) v^{i+1,\pi} (s') \enspace .
\end{align*}
As we only added a \emph{fixed} term to the reward, trivially, under the optimal Bellman operator $\mathcal{T} v^\pi$ we have that $\mathcal{T} v^\pi \geq v^\pi$. Also, this has a single fixed-point $\pi^*$:

\begin{align*}
    \tilde v^{i, \pi^*_{\gamma_1}}_{\gamma_2} \!&=\! r (s)\! +\! \left(\gamma_1 \!-\! \gamma_2\right)\!\! \sum_{s' \in \mathcal{S}} P^{\pi^*_{\gamma_1}} (s' | s, \pi^*_{\gamma_1} (s)) v^{i+1,\pi^*_{\gamma_1}} (s') \\
    &\enspace\enspace+ \gamma_2 \sum_{s' \in \mathcal{S}} \mathcal{P}(s'| s, \pi^*_{\gamma_1} (s)) v^{i+1,\pi^*_{\gamma_1}} (s') \numberthis\\
    &= r (s) \!+\! \gamma_1 \!\!\sum_{s' \in \mathcal{S}} \mathcal{P}(s'| s, \pi^*_{\gamma_1} (s)) v^{i+1,\pi^*_{\gamma_1}} (s') = v^{i,\pi^*}
\end{align*}

the above holds for any $\gamma_1, \gamma_2 \in [0, 1]$.
\end{proof}

\begin{theorem}
    For any $\gamma \in (0, 1]$, there exists a mapping $\tilde r (s,t)$ such that for any two trajectories $\tau_i, \tau_j$ and ${\mathbb{E}_{\tau^i} [\sum_{t=0}^T r_t] \geq \mathbb{E}_{\tau^j} [\sum_{t=0}^T r_t]}$ then ${\mathbb{E}_{\tau^i} [\sum_{t=0}^T \gamma^t \tilde r_t] \geq \mathbb{E}_{\tau^j} [\sum_{t=0}^T \gamma^t \tilde r_t]}$.
\end{theorem}

\begin{proof}
    While the problem is ill-defined, e.g., there exists many reward signals $\tilde r$ which satisfy this condition, finding one such reward is sufficient for existence.
    
    Consider the reward $\tilde r(s,t) = \frac{1}{\gamma^t} \mathbb{E}[r(s)]$. For any two trajectories $\tau_i, \tau_j$ the following holds:
    \begin{equation}
        \sum_{\tau_i} \gamma^t \tilde r(s,t) = \sum_{\tau_i} \gamma^t \frac{1}{\gamma^t} \mathbb{E}[r(s)] = \sum_{\tau_i} \mathbb{E}[r(s)] \, ,
    \end{equation}
    hence, if ${\mathbb{E}_{\tau^i} [\sum_{t=0}^T r_t] \geq \mathbb{E}_{\tau^j} [\sum_{t=0}^T r_t]}$ then ${\sum_{\tau_i} \gamma^t \tilde r_t \geq \sum_{\tau_j} \gamma^t \tilde r_t}$.
\end{proof}

\begin{prop}
    \cref{thm: reward ordering exists} does not hold for $\gamma = 0$.
\end{prop}

\begin{proof}
    Consider 3 trajectories:
    \begin{align*}
        &\tau_1 : s_A \rightarrow s_B \rightarrow s_C \\
        &\tau_2 : s_A \rightarrow s_B \rightarrow s_D \\
        \tau_3 : s_A \rightarrow s_E
    \end{align*}
    where the rewards are
    \begin{align*}
        r(s_B) = 0, r(s_C) = 2, r(s_D) = -1, r(s_E) = 1
    \end{align*}
    thus the following ordering holds $\tau_1 > \tau_3 > \tau_2 $.
    
    However, by setting $\gamma = 0$ the utility of each trajectory in the $\gamma$-discounted setting is defined as
    \begin{align*}
        &U(\tau_1) = \tilde r(s_B) + \gamma \tilde r(s_C) \stackrel{\gamma=0}{=} \tilde r(s_B)\\
        &U(\tau_2) = \tilde r(s_B) + \gamma \tilde r(s_D) \stackrel{\gamma=0}{=} \tilde r(s_B)\\
        U(\tau_3) = \tilde r(s_E)
    \end{align*}
    thus, $U(\tau_1) = U(\tau_2)$ and the ordering $U(\tau_1) > U(\tau_3) > U(\tau_2) $ can not be maintained.
\end{proof}

\begin{lemma}
    For any MDP $(\mathcal{S}, \mathcal{A}, \mathcal{R}, \mathcal{P}, T)$, the surrogate reward $\tilde r$ for the $\gamma$-discounted problem is not unique.
\end{lemma}

\begin{proof}
This can be shown trivially, by observing that any multiplication of the rewards $\mathcal{R}$ by a positive (non-zero) scalar, is identical. Meaning that, $\forall \pi \in \Pi, s \in \mathcal{S}, \alpha \in \mathbb{R}^+$:
\begin{align*}
    v^\pi_{\alpha \mathcal{R}} (s) &= \mathbb{E}^\pi [\sum_{t=0}^\infty \gamma^t \alpha r_t | s_0 = 0] \\
    &= \alpha \mathbb{E}^\pi [\sum_{t=0}^\infty \gamma^t r_t | s_0 = 0] = \alpha v^\pi_{\mathcal{R}} (s) \enspace .
\end{align*}
\end{proof}

\begin{prop}
    If $\tilde r$ enables recovering the optimal policy, e.g., \cref{eqn: existance}, and $\max_{s,a} |\mathcal{\hat P}(s,a) - \mathcal{P}(s,a)| \leq \epsilon_P$, where $\hat P$ is the empirical probability transition estimates, then
    \begin{equation*}
        \forall s \in \mu : || v^\pi_{\hat P} - v^\pi_P ||_\infty  \leq \frac{\gamma (1 - \gamma^T) \tilde R_\text{max}}{2 (1-\gamma)^2} \epsilon_P \enspace ,
    \end{equation*}
    where $\tilde R_\text{max} = \max_{s \in \mathcal{S}} \tilde r(s)$.
\end{prop}

\begin{proof}
The proof follows that of the Simulation Lemma, with a slight adaptation for the finite horizon discounted scenario with zero reward error.
$\forall s \in \mathcal{S}:$
\begin{align*}
    &\left|v_{\hat P}^\pi (s) - v_P^\pi (s)\right| = \gamma \left| r(s) + \sum_{s' \in \mathcal{S}} \hat P (s' | s, \pi) v_{\hat P}^\pi (s') - r(s) - \sum_{s' \in \mathcal{S}} P (s' | s, \pi) v_P^\pi (s') \right| \\
    &= \gamma \left| \sum_{s' \in \mathcal{S}} \left(\hat P (s' | s, \pi) v_{\hat P}^\pi (s') - P (s' | s, \pi) v_P^\pi (s') \right) \right| \\
    &= \gamma \left| \sum_{s' \in \mathcal{S}} \left(\hat P (s' | s, \pi) v_{\hat P}^\pi(s') - P (s' | s, \pi) v_{\hat P}^\pi(s') + P (s' | s, \pi) v_{\hat P}^\pi(s') - P (s' | s, \pi) v_P^\pi(s') \right) \right| \\
    &\leq \gamma \left| \sum_{s' \in \mathcal{S}} \left(\hat P (s' | s, \pi) v_{\hat P}^\pi(s') - P (s' | s, \pi) v_{\hat P}^\pi(s') \right) \right| + \gamma || v_{\hat P}^\pi - v_P^\pi ||_\infty \\
    &= \gamma \left| \sum_{s' \in \mathcal{S}} \left(\hat P (s' | s, \pi) - P (s' | s, \pi) \right) \left(v_{\hat P}^\pi(s') - \frac{\tilde R_\text{max}(1 - \gamma^T)}{2(1-\gamma)} \right) \right| \\
    &+ \gamma || v_{\hat P}^\pi - v_P^\pi ||_\infty \leq \gamma \epsilon_P \frac{\tilde R_\text{max}(1 - \gamma^T)}{2(1-\gamma)} + \gamma || v_{\hat P}^\pi - v_P^\pi ||_\infty \enspace ,
\end{align*}

where $\sum_{t=0}^T \gamma^t \tilde r_t \leq \tilde R_\text{max} (1 - \gamma^T) / (1-\gamma)$. Moreover, in order to recover the bounds for the original problem, plugging in Theorem \ref{lemma: existence} we retrieve the original value and thus $\sum_{t=0}^T \gamma^t \tilde r_t \leq T \cdot R_\text{max}$; in which the robustness is improved by a factor of $(1-\gamma)^{-1}$.
\end{proof}

As we are analyzing the finite-horizon objective, it is natural to consider a scenario where the uncertainty in the transitions $|\hat P(s' | s, a) - P(s' | s, a)|$ is concentrated at the end of the trajectories. This setting is motivated by practical observations -- the agent commonly starts in the same set of states, hence, most of the data it observes is concentrated in those regions leading to larger errors at advanced stages.

The following proposition presents how the bounds improve when the uncertainty is concentrated near the end of the trajectories, as a factor of both the discount $\gamma$ and the size of the uncertainty region defined by $\mathcal{S}_L$.

\begin{prop}
    We assume the MDP can be factorized, such that all states $s$ that can be reached within $T-L \leq t \leq T$ steps are unreachable for $t < T-L$, $\forall \pi \in \Pi$. We denote the ${T-L \leq t \leq T}$ reachable states by $\mathcal{S}_L$.
    
    If $\max_{s \in \mathcal{S}_L, a} ||\mathcal{\hat P}(s,a) - \mathcal{P}(s,a)||_1 \leq \epsilon_P$ and $\max_{s \in \mathcal{S}\SLASH\mathcal{S}_L, a} ||\mathcal{\hat P}(s,a) - \mathcal{P}(s,a)||_1 = 0$ then
    \begin{equation*}
        \forall s \in \mu : || v^\pi_{\hat P} - v^\pi_P ||_\infty  \leq \gamma^{T-L} \frac{\gamma (1 - \gamma^L) \tilde R_\text{max}}{2 (1-\gamma)^2} \epsilon_P \enspace .
    \end{equation*}
\end{prop}

\begin{proof}
    Notice that this formulation defines a factorized MDP, in which we can analyze each portion independently.
    
    The sub-optimality of $|v^\pi_{\hat M} (s) - v^\pi_M (s) |$ for $s \in \mathcal{S}_L$ can be bounded similarly to Proposition \ref{lemma: simulation}
    \begin{equation*}
        \forall s \in \mathcal{S}_L: || v^\pi_{\hat P} - v^\pi_P ||_\infty  \leq \frac{\gamma (1 - \gamma^{L}) \tilde R_\text{max}}{2 (1-\gamma)^2} \epsilon_P \enspace ,
    \end{equation*}
    where the horizon is $L$ due to the factorization of the MDP.
    
    For $s \in \mathcal{S}\SLASH\mathcal{S}_L$, clearly as the estimation error is $0$ for states near the initial position, then the overall error in these states is bounded by $\gamma^{T-L} \frac{\gamma (1 - \gamma^{L}) \tilde R_\text{max}}{2 (1-\gamma)^2} \epsilon_P$ where $\gamma^{T-L}$ is due to the $T-L$ steps it takes to reach the uncertainty region.
\end{proof}

\section{Extended Results}
\subsection{Experimental Details}
We compare reward tweaking with the baseline TD3 model, adapted to the finite horizon setting. The results are presented graphically in \cref{fig: hopper pareto} and numerically in the supplementary material. Our implementations are based on the original code by \citet{fujimoto2018addressing}. In their work, as they considered stationary models, they converted the task into an infinite horizon setting, by considering ``termination signals due to timeout" as non-terminal. However, they still evaluate the policies on the total reward achieved over $1,000$ steps.

Although this works well in the MuJoCo control suite, it is not clear if this should work in the general case, and whether or not the algorithm designer has access to this knowledge (that termination occurred due to timeout). Hence, we opt for a \emph{finite-horizon non-stationary actor and critic scheme}. To achieve this, we concatenate the normalized remaining time $\tilde t = (T-t)/T$ to the state.

As we are interested in a finite-horizon task, the reward tweaking experiments are performed on this non-stationary version of TD3. In our experiments, $\tilde r$ and $\pi$ are learned in parallel (\cref{fig: reward tweaking}). The reward tweaker $\tilde r$, is represented using 4 fully-connected layers ($|s| \rightarrow 256 \rightarrow 256 \rightarrow 256 \rightarrow 1$) with ReLU activations. Training is performed on trajectories, which are sampled from the replay buffer, every $4$ time-steps. Each time, the loss is computed using $256$ trajectories, each up to a length of $100$ steps. Due to the structure of the environment, timeout can occur at any state. Hence, given a sampled sub-trajectory, we update the normalized remaining time of the sampled states, based on the length of the sub-trajectory, i.e., as if timeout occurred in the last sampled state. This enables to train $\tilde r$ using partial trajectories, while the model is required to extrapolate the rest. Our computing infrastructure consists of two machines with two GTX 1080 GPUs, each.

\begin{figure}[H]
    \centering
    \includegraphics[width=0.49\textwidth]{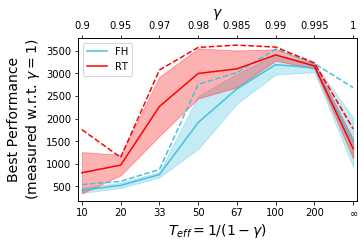}
    \caption{Hopper-v2 experiments, comparing the TD3 finite horizon baseline [FH] with our method, Reward Tweaking [RT]. The top X axis denotes the discount factor $\gamma$, and the bottom presents the effective planning horizon $T_{eff} = \frac{1}{1-\gamma}$. The dashed line shows the max performance over all seeds, the solid is the mean and the colored region is the 95\% confidence bound.}
\end{figure}

\begin{table}[H]
    \centering
    % \setlength\tabcolsep{3.1pt}
    % \scriptsize
    \fontsize{7.6}{9}
    \begin{tabularx}{0.822\linewidth}{c|c|c|c|g|g}
         & & $\gamma = 1$ & $\gamma = 0.995$ & $\gamma = 0.99$ & $\gamma = 0.985$ \\
        \hline\hline
        \multirow{2}{*}{\thead{Baseline}} & \textbf{Avg.} & $\mathbf{1474 \pm 608}$ & $\mathbf{3120 \pm 106}$ & $3190 \pm 246$ & $2656 \pm 369$ \\
        & \textbf{Max} & $2682$ & $3215$ & $3536$ & $3015$ \\
        \hline
        \multirow{2}{*}{\thead{Reward Tweaking}} & \textbf{Avg.} & $\mathbf{1331 \pm 31}$ & $\mathbf{3161 \pm 73}$ & $\mathbf{3405 \pm 141}$ & $\mathbf{3097 \pm 454}$ \\
        & \textbf{Max} & $1771$ & $3228$ & $3579$ & $3622$ \\
        \hline
    \end{tabularx}
    ~\\~\\~\\
    \centering
    \begin{tabularx}{0.8\linewidth}{c|c|g|g|c|c}
         & & $\gamma = 0.98$ & $\gamma = 0.97$ & $\gamma = 0.95$ & $\gamma = 0.9$ \\
        \hline\hline
        \multirow{2}{*}{\thead{Baseline}} & \textbf{Avg.} & $1912 \pm 654$ & $678 \pm 6$ & $523 \pm 66$ & $414 \pm 63$ \\
        & \textbf{Max} & $2755$ & $874$ & $612$ & $541$ \\
        \hline
        \multirow{2}{*}{\thead{Reward Tweaking}} & \textbf{Avg.} & $\mathbf{2996 \pm 617}$ & $\mathbf{2266 \pm 730}$ & $\mathbf{970 \pm 253}$ & $\mathbf{801 \pm 510}$ \\
        & \textbf{Max} & $3572$ & $3071$ & $1140$ & $1754$ \\
        \hline
    \end{tabularx}
    \caption{Numerical results for the Hopper experiments. For each experiment, we present the average score between the best results of the various seeds, the standard deviation and the best (max) performance across all seeds. For each $\gamma$, we highlight the models that, with high confidence, based on the Avg. and STD, performed the best.}
    \label{tab: hopper table}
\end{table}

\end{document}